\documentclass{article}
\usepackage{ijcai24}

\usepackage{times}
\usepackage{soul}
\usepackage{url}
\usepackage[hidelinks]{hyperref}
\usepackage[utf8]{inputenc}
\usepackage[small]{caption}
\usepackage{graphicx}
\usepackage{amsmath}
\usepackage{amsthm}
\usepackage{booktabs}
\usepackage{algorithm}
\usepackage{algorithmic}
\usepackage[switch]{lineno}

\newtheorem{theorem}{Theorem}

\usepackage{enumerate}
\usepackage{amsmath,amssymb,amsfonts}
\newtheorem{lemma}{Lemma}
\newtheorem{definition}{Definition}
\newcommand{\myref}[1]{Eq.~(\ref{#1})}
\usepackage{subfigure}

\usepackage{soul, color, xcolor}
\definecolor{myColor}{rgb}{0,0,0} 
\makeatletter
\newcommand*{\new}{\@ifnextchar\bgroup{\new@}{\color{myColor}}}
\newcommand*{\new@}[1]{{\textcolor{myColor}{#1}}}
\makeatother

\usepackage{fancyhdr}
\title{A First Step Towards Runtime Analysis of\\ Evolutionary Neural Architecture Search}

\author{
Zeqiong Lv$^1$
\and
Chao Qian$^2$\and
Yanan Sun$^{1*}$
\affiliations
$^1$College of Computer Science, Sichuan University, Chengdu 610065, China\\
$^2$National Key Laboratory for Novel Software Technology, and School of Artificial Intelligence, Nanjing University, Nanjing 210023, China\\
\emails
zq\_lv@stu.scu.edu.cn,
qianc@nju.edu.cn,
ysun@scu.edu.cn
}

\begin{document}

\maketitle

\begin{abstract}

    Evolutionary neural architecture search (ENAS) employs evolutionary algorithms to find high-performing neural architectures automatically, and has achieved great success. However, compared to the empirical success, its rigorous theoretical analysis has yet to be touched. This work goes preliminary steps toward the mathematical runtime analysis of ENAS. In particular, we define a binary classification problem $\textsc{UNIFORM}$, and formulate an explicit fitness function to represent the relationship between neural architecture and classification accuracy. Furthermore, we consider (1+1)-ENAS algorithm with mutation to optimize the neural architecture, and obtain the following runtime bounds: \new{both the local and global mutations find the optimum in an expected runtime of $\Theta(n)$, where $n$ is the problem size}.
    The theoretical results show that \new{the local} and \new{global} mutations achieve nearly the same performance on $\textsc{UNIFORM}$. 
    Empirical results also verify the equivalence of these two mutation operators. 

\end{abstract}

\section{Introduction}

Manually designing deep neural networks (DNNs) requires expertise in both DNNs and domain knowledge of the problem to be solved. To address this, researchers propose neural architecture search (NAS)~\cite{elsken2019neural} that can automatically design promising neural architectures. In the early stages, NAS often employed the RL algorithms~\cite{zoph2016neural}. In recent years, the gradient algorithms~\cite{liu2018darts} have been proposed to improve search efficiency. Evolutionary computation (EC) is another optimization algorithm widely used to solve NAS. EC-based NAS (also known as evolutionary NAS, i.e., ENAS) regards the neural architectures as individuals and then searches for promising neural architecture with evolutionary operators, such as mutation. In fact, EC has already been used to optimize both weights and architectures of (shallow) neural networks since the late 1980s~\cite{montana1989training,yao1999evolving,stanley2002evolving}, which is also known as ``evolving neural networks" or ``neuroevolution." But in recent years, along with the DNNs' breakthrough, ENAS is increasingly popular in the field of automated machine learning and has yielded impressive results of (deep) neural architectures in DNNs~\cite{real2019regularized,sun2019completely,liu2021survey}. 

Along with extensive empirical knowledge on EC solving the complex optimization problem, there is also increasing solid theory that guides the design and application of EC. In particular, runtime analysis, which represents the expected number of fitness evaluations until an optimal or approximate solution is searched for the first time, has become an established branch in the theory of EC that enables such results~\cite{auger2011theory,neumann2013bioinspired,zhou2019evolutionary,doerr2020theory}. This is reflected in that runtime bounds (lower and upper bounds) can help us understand the influence of certain parameter choices like the mutation rate, the population size, or the criterion selection. The tighter the runtime bounds are, the more can be said about these choices. 
  
Analyzing the runtime bounds of ENAS in a rigorous and precise manner, however, is a challenging task and has yet to be touched. One of the main difficulties is the lack of an explicit expression of the fitness function, which in turn hinders the utilization of mathematical runtime analysis tools. Typically, fitness calculations in ENAS are obtained by training neural architectures, which is $NP$-hard when the neural architecture has more than two neurons~\cite{blum1988training}. Such events make it impractical to comprehend the quality of the neural architecture and partition solution space from a theoretical perspective. Therefore, it is far beyond our current capability to analyze the runtime of ENAS on real-world problems related to NAS. It is worth noting that encountering such difficulty is common in theoretical analysis. Theoreticians often resort to analyzing search strategies on artificial problems with general structures, such as \textsc{OneMax} and \textsc{LeadingOnes}~\cite{droste2002analysis}, hoping that this gives new insight into the behavior of such search algorithms on real-world problems. This inspires us to provide an explicit fitness function on an artificial problem that maps the NAS problem. Note that Fischer \emph{et al.}~\shortcite{fischer2023first} have conducted the first runtime analysis of neuroevolution, which, however, focuses on optimizing weights of neural networks given an architecture instead of searching neural architectures.

When tackling new problems, EC runtime studies typically begin with general-purpose algorithms, such as the (1+1)-evolutionary algorithm (EA), that restricts the size of the population to just one individual and does not use crossover. Researchers concentrate on one-bit and multi-bit mutations, and have conducted rigorous runtime analysis on the influence of mutation operators~\cite{doerr2006speeding,doerr2008comparing,kotzing2012max}. Usually, one-bit mutation is considered as local search, whereas multi-bit mutation (e.g., bit-wise mutation) relates to global search since it can increase the step size to prevent a local trap. One might be tempted to believe that the global operator generally is superior to the local one. However, \cite{doerr2008comparing} showed that bit-wise mutation could be worse than one-bit mutation \new{hillclimbing (e.g., \textsc{OneMax}, \textsc{LeadingOnes})}. Such work can provide insights into the choice of mutation operators. This is one important motivation for comparing the performance of these two mutations in a rigorous way. 

As indicated above, the focus of this paper is to take the first step towards the runtime analysis of ENAS algorithms. The contributions are summarized as follows.
\begin{itemize}
	\item We define a binary classification problem \textsc{Uniform}, and formulate an explicit fitness function to measure the precise quality (classification accuracy) of any given neural architecture.
	\item We partition the solution space based on the fitness, which is a key step in carrying out the runtime analysis by using mathematical tools.
	\item We analyze the runtime bounds of the (1+1)-ENAS algorithms with \new{local} and \new{global} mutations, and prove the equivalence of the two mutations on \textsc{Uniform}. Empirical results also verify the theoretical results. 
\end{itemize}

\section{Preliminaries} \label{sec:Preliminaries}



This section presents the foundations of neural architectures for classification problems that are relevant to our study and introduces a (1+1)-ENAS algorithm for searching neural architecture. \new{Note that for intervals of integers we write $[k..l]:=\{x\in\mathbb{Z}\mid k\leq x\leq l\}$.}

\subsection{Neural Architecture}	\label{sec:DNN}
In the work of~\cite{fischer2023first}, they introduced a feed-forward neural architecture that consists of an input layer, hidden layers, and an output layer, to solve binary classification problems. The hidden layers are constructed by a set of basic units (each unit is a single neuron with binary threshold activation function) that output binary strings, and an OR neuron that computes a Boolean OR of the binary outputs from the previous layer. Such a unit with a single neuron can act as a hyperplane that points to a half-space (a linearly divisible region)\footnote{A hyperplane is defined as the set $\{\mathbf{x}\in \mathbb{R}^D \vert \mathbf{ax}= z\}$, where $\mathbf{a}$ is a nonzero vector in $\mathbb{R}^D$ and $z$ is a scalar, and a half-space is defined as the set $\{\mathbf{x}\in \mathbb{R}^D \vert \mathbf{ax}\geq z\}$~\cite{bertsimas1997introduction}. Note that a hyperplane is the boundary of a corresponding half-space.}. Thus, only the linear classification problem is touched upon.

However, most classification problems are non-linearly divisible and can be described as the disjoint union of polyhedra of $\mathbb{R}^D$, where a polyhedron is the intersection of a finite number of half-spaces, categorized into bounded \new{(finite region formed by hyperplanes)} and unbounded polyhedra~\cite{bertsimas1997introduction}. In Section~\ref{sec:problem_definition}, we will define a binary classification problem with classification regions: half-space, bounded polyhedron, and unbounded polyhedron.

\begin{figure}
	\centering
	\subfigure[A-type block]{
		\includegraphics[height=0.65in]{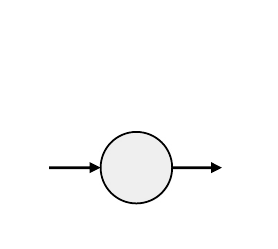}
		\label{fig:block_a}
	} 
	\subfigure[B-type block]{
		\includegraphics[height=0.65in]{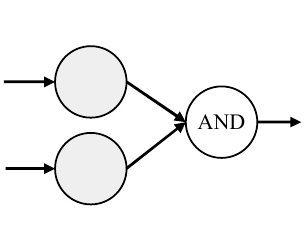}
		\label{fig:block_b}
	} 
	\subfigure[C-type block]{
		\includegraphics[height=0.65in]{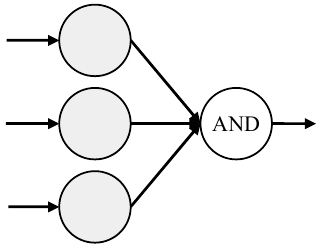}
		\label{fig:block_c}
	} 
	\caption{ Three types of blocks that can be used to build DNNs with different topological architectures. All of the neurons output Boolean value. The latter two blocks are connected with ADD neurons and can function as non-linear classifiers. }
	\label{fig:blocksType}
\end{figure}

By adding a layer to compute the Boolean AND of the outputs of finite neurons, the unit with these finite neurons and the added layer can point to any polyhedra (with nonlinear regions included). Note that polyhedra intersected by one to three half-spaces can co-form any polyhedron. Thus, three kinds of units are sufficient to build a neural architecture that can classify most classification problems. Typically, unit is also called block in the NAS community. These three types of blocks are shown in Figure~\ref{fig:blocksType}. The neural architecture with an A/B-type block can classify unbounded polyhedra, while the neural architecture with a C-type block is suitable for bounded polyhedra. Note that when such neural architecture is used to deal with specific classification problems, one needs to consider which type of blocks to use and what their number is, depending on the polyhedra of the problem.

\subsection{ENAS Algorithm} \label{sec:enas_fitness}

ENAS aims to search for a promising neural architecture from a given search space. We abbreviate the number of blocks of types A, B, and C to $n_A$, $n_B$ and $n_C$, where $n_A,n_B,n_C\in\mathbb{Z}$. \new{A solution can be represented as $(n_A,n_B,n_C)$.} The search space can be represented as \new{$\mathbb{Z}\times \mathbb{Z}\times \mathbb{Z}$}. In this study, we consider the (1+1)-ENAS algorithm with mutation (i.e., based on (1+1)-EA), whose main steps are: 
\begin{enumerate}[1)]
	\item \textbf{Individual Initialization}: Randomly initialize the number of blocks in the initial solution $x_0$ according to the uniform distribution $U[0,s]$, where $s\in \mathbb{N}^+$ is a predetermined initial upper bound for each type of block.
	\item \textbf{Offspring Generation}: Generate $y_t$ by executing the mutation operation on parent $x_t$.
	\item \textbf{Environmental Selection}: Select the solution with higher fitness value in $x_t$ and $y_t$ to enter into the next iteration. \new{Note that, the offspring is preferred to the parent in case of ties by the selection operator.}
\end{enumerate}

For each mutation operation in offspring generation, it applies the \emph{Addition, Removal} or \emph{Modification} operator (used in AE-CNN~\cite{sun2019completely}) uniformly at random, and repeats this process $K$ times independently. Specifically, it updates the individual by randomly adding a block of type A/B/C, removing a block of type A/B/C, or modifying a block to one of the other types. Note that for the empty blocks, the removal operator will keep it unchanged. If the (1+1)-ENAS algorithm adopts the \new{local} mutation, $K$ will be set as 1. Otherwise, if the (1+1)-ENAS algorithm adopts the \new{global} mutation, it will mutate $K\geq 1$ times in an iteration, where $K$ is determined by the Poisson distribution with $\lambda=1$, i.e., $K-1\sim Pois(1)$. Such mutation behavior is inspired by the mutation operator in \new{fixed parameter evolutionary algorithm~\cite{kratsch2010fixed},} genetic programming~\cite{durrett2011computational,qian2015variable}, and evolutionary sequence selection~\cite{qian2018sequence,qian2023multi}.

\section{Constructing Analyzable Problem} \label{sec:UNIFORM_problem}


In the machine learning community, classification problems are a crucial topic. Among them, the classic binary classification problems are commonly employed for comparing the performance of different DNNs. By establishing a mathematical relationship between the configuration of DNNs and classification accuracy on a binary classification problem, an explicit fitness function for evaluating DNNs on this problem can be formulated. Recently, Fischer \emph{et al.}~\shortcite{fischer2023first} formulated the explicit fitness function between parameters (i.e., weights and biases) of the neuron and classification accuracy on a binary classification problem with labels in $\{0,1\}$ on a unit hypersphere, and used it to analyze the runtime of (1+1)-Neuroevolution algorithm. However, the binary classification problem they defined restricts its solution scope to a fixed two-layer neural network, making it unsuitable for studying ENAS algorithms. Nevertheless, they motivated us to update their binary classification problem while formulating an explicit fitness function between neural architecture and classification accuracy. Furthermore, we can start analyzing the defined problem by partitioning the solution space based on fitness, which will be used in the runtime analysis.

\subsection{Problem Definition} \label{sec:problem_definition}

We consider points on $S:=\{x\in \mathbb{R}^D \vert \|x\|_2 \leq 1 \}$ as inputs to the neural network, and define a binary classification problem \textsc{Uniform} with labels in $\{0, 1\}$ \new{(i.e., negative class and positive class)} on these points. Although the interpretation of neural network classifying points as 0 or 1 does not depend on $D$, as preliminary work, we will study the case $D=2$ for simplicity. \new{Let evenly divide the unit circle $S$ into $n$ sectors with angle $2\pi/n$. We denote the $k$-th sector as $S_{\mathrm{sec}}^k$. Thus, $S=\cup_{k=1}^{n}S_{\mathrm{sec}}^k$. Furthermore, let each sector $S_{\mathrm{sec}}^k$ be divided into a triangle $S_{\mathrm{tri}}^k$ with area $\frac{1}{2}\sin(\frac{2\pi}{n})$ and a segment region (a segment is cut from a circle by a ``chord'') $S_{\mathrm{sec}}^k\setminus S_{\mathrm{tri}}^k$. Then, we define the \textsc{Uniform} problem by assigning the positive and negative class points as shown in Definition~\ref{def:uniform_problem}.}
\begin{definition} [\textsc{Uniform}] \label{def:uniform_problem}
    {
    \color{myColor}
    The problem assigns positive labels to points within
        $ S^+=
        (\cup_{k \in [1..{n}/{4}]}$ $S_{\mathrm{tri}}^{4k-3}) 
        \cup( \cup_{k \in [1..{n}/{2}]} S_{\mathrm{sec}}^{2k} \setminus S_{\mathrm{tri}}^{2k}) 
        \cup( \cup_{k \in [1..{n}/{4}]} S_{\mathrm{sec}}^{4k-1})
        $
    ,
    where $n/4\in \mathbb{Z}, n\geq 6$, while labeling points in $S\setminus S^+$ as negative class.
    }
\end{definition}

\begin{figure} [h]
	\centering
	\subfigure[]{
		\includegraphics[width=0.35\linewidth]{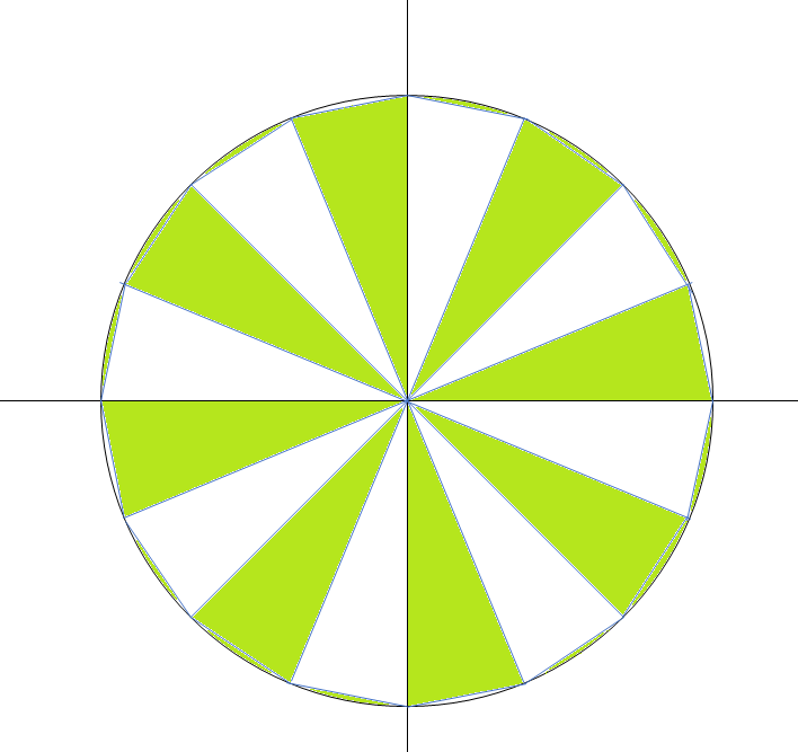}
		\label{fig:problem_UNIFORM}
	} 
	\subfigure[]{
		\includegraphics[width=0.35\linewidth]{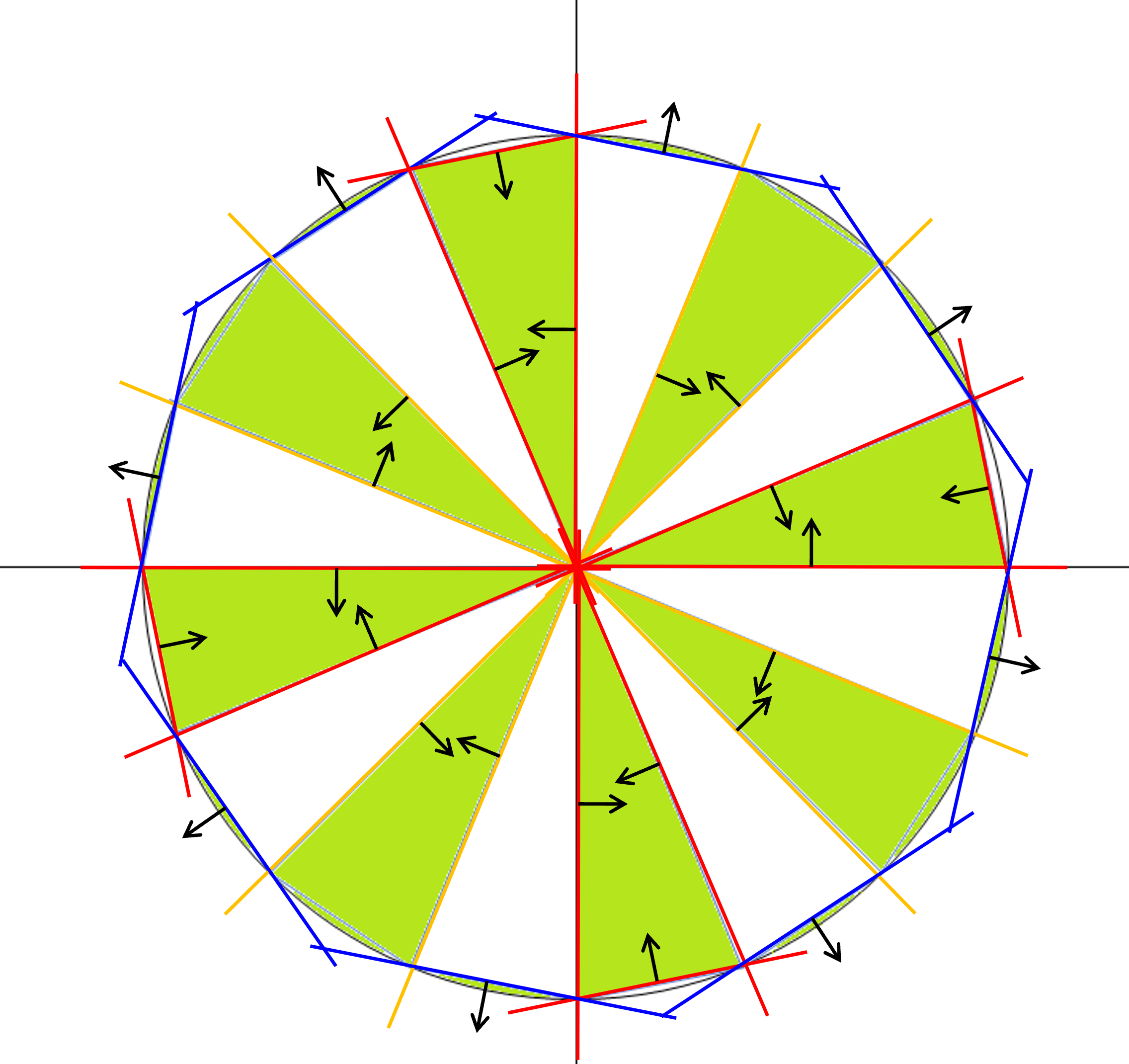}
		\label{fig:optimum}
	}
	\setlength{\abovecaptionskip}{-0.2em} 
	\caption{(a) Illustration of \textsc{Uniform} ($n=16$) when $D=2$. The green regions constitute the positive  points classified as 1, while the white regions constitute the negative points classified as 0. (b) One optimal solution with a fitness of 1. The blocks at optimal positions (with the optimal weights) are colored as follows: blue for A-type blocks, yellow for B-type blocks, and red for C-type blocks. The arrow points to the positively classified half-space.}
\end{figure}

The sketch of \textsc{Uniform} ($n=16$) is shown in Figure~\ref{fig:problem_UNIFORM}. Visually, the positive  points fall into three cases of regions: 1) the green (positive) segment region connected to the white (negative) triangle, which can be seen as a half-space; 2) the green sector region, which can be seen as an unbounded polyhedron; 3) the green triangle  connected to the white segment region, which can be seen as a bounded polyhedron. Blocks of types A, B, and C can correctly classify these three cases respectively. We use $a, b, c$ to represent the number of these three regions respectively. Then, we have $a=n/2, b=c=n/4$ for \textsc{Uniform}. 

As discussed in Section~\ref{sec:DNN}, there is an optimal configuration for blocks in neural architecture: $n_A, n_B,$ and $n_C$ are equal to $a, b,$ and $c$ respectively. Refer to Figure \ref{fig:optimum} for a visual representation of an optimal solution. Note that the optimal solution is not unique because the combination of the A-type block and C-type block (i.e., ``A+C"-type blocks) leads to the same fitness improvement as the B-type block. In practice, the optimal solution simultaneously satisfies the following threshold conditions: $n_A\geq a+\max\{0,(b-n_B)\}, n_B+n_C\geq b+c, n_C\geq c$. If $\max\{0,(b-n_B)\}=(b-n_B)>0$, it means that the optimal solution only has $n_B<b$ B-type blocks and uses $(b-n_B)$ ``A+C"-type blocks (i.e., $(b-n_B)$ A-type blocks and $(b-n_B)$ C-type blocks) to make up for it. The threshold conditions will serve as guidelines for the theoretical analysis presented below.

\subsection{Fitness Function Explicitization} \label{sec3:fitness_ij}

The fitness function $f$ used to evaluate \new{the classification accuracy of} neural architectures can be interpreted as the fraction of correctly classified points on the unit \new{circle} $S$, which is similar to the definition of $f$ in~\cite{fischer2023first}. Then, we have
\begin{equation}
	\label{eq:fitness_function}
	f = \frac{{\text{{vol}} 
			\left(
			(C \cap L) \cup (\overline{C} \cap \overline{L})
			\right)}}
	{{\text{{vol}}(S)}},
\end{equation}
where $C$ is the union of polyhedrons pointed by blocks, $\overline{C}=\mathbb{R}^2 \backslash C$, $L\subset S$ represents the set of points classified as 1, $\overline{L}=\mathbb{R}^2 \backslash L$, and $\text{{vol}}(\cdot)$ denotes the hypervolume. The union of $(C \cap L)$ and $(\overline{C} \cap \overline{L})$ represents the hypervolume of instances correctly classified as 1 or 0. 

The explicit expression of~\myref{eq:fitness_function} relies on the parameters of each neuron. \new{Note that each neuron with binary threshold activation can be seen as a hyperplane (line). Thus, the parameters that needed to be tuned are the angle $\varphi$ (weight) of the unit normal vector of each hyperplane and distance (bias) of each hyperplane from the origin~\cite{fischer2023first}.} However, ENAS primarily focuses on the evolution of neural architecture, leaving the optimization of parameters to alternative training approaches like gradient descent. The crucial consideration is theoretically guaranteeing the \new{optimal parameters}. For this reason, we impose a constraint on the training of parameters (shown in Definition~\ref{def:fixed_parameters_strategy}) to simplify the fitness evaluation progress in ENAS. \new{In fact, due to the finite training set, it will inevitably introduce noise (i.e., cannot guarantee the optimal parameter) when evaluating the performance of the architecture. However, considering a noisy fitness function is challenging, as it relies on an unknown finite training set and requires to analyze the distribution of the dataset. A noisy fitness function is subject for future research.}

\begin{definition} [Neurons Parameter Training Strategy]
	\label{def:fixed_parameters_strategy}
    The bias of the neuron in the A-type block is set as $\cos(\pi/n)$; the bias of the neurons in the B-type block is set as 0, and \new{the difference in $\varphi$ between these neurons is} $(2\pi/n)$; two neurons in C-type blocks have a bias of 0, with \new{$\varphi$} difference of $(2\pi/n)$, while the third neuron has a bias of $\cos(\pi/n)$ and a \new{$\varphi$} difference with the other two neuron of $(\pi/n)$. \new{Then, a training approach such as gradient descent is used to theoretically guarantee the optimal parameters.}
\end{definition}

Furthermore, based on \myref{eq:fitness_function} and Definition~\ref{def:fixed_parameters_strategy}, we give Lemma~\ref{lemma:fitness} to portray how the basic units of neural architecture (i.e., blocks) influence the fitness value. The formulation \myref{eq:fitness} is an explicit fitness function that can directly calculate the classification accuracy for any given architecture.


\begin{lemma}
	\label{lemma:fitness}
	Let $x=\{n_A,n_B,n_C\}$ be a solution on~\textsc{Uniform}. Let $i=\min\{{n_B+n_C,b+c}\}$ ($i\in \{0,\dots,b+c\}$) be the number of B/C-type blocks that can classify the {green triangles}, and $j=\min\{n_A, {a+\max\{0,b-n_B\}}\} + \min\{n_B,b\} - \max\{0,\min\{(n_B-b),(c-n_C)\}\}$ ($j\in\{0,\dots,a+b\}$) be the number of A/B-type blocks that classify {the green segment regions} minus the number of B-type blocks that classify the {white segment regions connected to green triangle}. Using the training strategy in Definition~\ref{def:fixed_parameters_strategy} for the architecture training, the fitness value of $x$ can be explicitly represented by
	\begin{equation}
		\begin{aligned}
			%
			%
			 f(x) & = \frac{1}{\pi} \cdot 
			 (
			     (   \frac{n}{2}+i)\cdot \text{the area of a triangle } \\ 
			     & \qquad + (\frac{n}{4}+j)\cdot \text{the area of a segment region} 
			 )
			 .
		\end{aligned}
		\label{eq:fitness}
	\end{equation}
\end{lemma}

{
\color{myColor}
To further understand how a neural architecture works on \textsc{Uniform}, we take an example to illustrate the fitness calculation of a given neural architecture. We take the problem shown in Figure~\ref{fig:problem_UNIFORM} as an example, i.e., $n=16$. Given a set of blocks including an A-type, a B-type, and a C-type, i.e., $n_A=n_B=n_C=1$, based on Section~\ref{sec:DNN}, the produced neural architecture $A_\mathrm{NA}$ is shown on Figure~\ref{fig:A_NA}. The best position of the hyperplanes in $A_\mathrm{NA}$ is shown in Figure~\ref{fig:A_NN}. Because $i=n_B+n_C=2$ and $j=n_A+n_B=2$, the fitness of architecture $A_\mathrm{NA}$ is $f(A_\mathrm{NA})=((\frac{n}{2}+i)\cdot \text{the area of a triangle} + (\frac{n}{4}+j)\cdot \text{the area of a segment region}) \cdot \frac{1}{\pi} = \sin(\frac{\pi}{8}) \cdot \frac{2}{\pi} + \frac{3}{8} \approx 0.62$, where the areas of a triangle and a segment region can be calculated by $\frac{1}{2} \sin(\frac{2\pi}{n})$ and $(\frac{\pi}{n} - \frac{1}{2}\sin(\frac{2\pi}{n}))$, respectively. Let the neural network with optimal parameters be denoted as $A_\mathrm{NN}$ which is also called a classifier for \textsc{Uniform} when $n=16$. The classification accuracy of $A_\mathrm{NN}$ is same as $f(A_\mathrm{NA})$ based on~\myref{eq:fitness_function}. Below we show how this classifier works. Given an input (positive point) whose coordinates are (0.5,0.6), the OR neuron receives a binary string ``010'' (where ``1'' is calculated by B-type block) and finally outputs 1, indicating that the classifier $A_\mathrm{NA}$ can correctly classify this point.

\begin{figure} [h]
	\centering
	\subfigure[]{
		\includegraphics[width=0.48\linewidth]{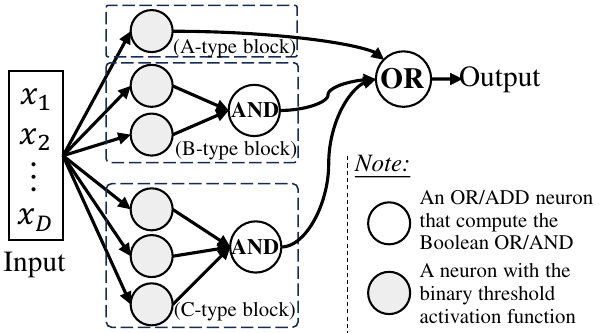}
        \label{fig:A_NA}
	} 
	\subfigure[]{
		\includegraphics[width=0.45\linewidth]{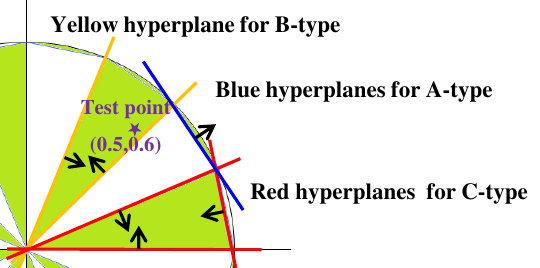}
        \label{fig:A_NN}
	} 
    \caption{(a) The neural architecture $A_\mathrm{NA}$ comprises an A-type, a B-type, and a C-type. (b) The best positions of the hyperplanes in $A_\mathrm{NA}$ for solving \textsc{Uniform} ($n=16$), i.e., the sketch of the classifier $A_\mathrm{NN}$. Given a point (0.5,0.6) as input, the classifier $A_\mathrm{NN}$ outputs 1.}
\end{figure}
}

Notably, using ENAS to solve the \textsc{UNIFORM} problem still primarily concentrates on the quantities of A, B, and C-type blocks within the neural network (i.e., $n_A$, $n_B$, and $n_C$), but simplifying the complexity of the problem from the perspective of fitness evaluation, making it resemble an optimization problem similar to the classical OneMax problem~\cite{droste2002analysis}. Unlike OneMax, searching for the optimum of \textsc{Uniform} is not just about increasing the number of specific bits, but also requires considering the potential negative impact on fitness when adding certain bits (details in Section~\ref{sec:subspace_analyse}). Moreover, the non-uniqueness of the problem in terms of optimal solutions is also a distinction and challenge.


\subsection{Solution Space Partition} \label{sec:subspace_analyse}
We partition the solution space based on fitness, which serves for the runtime analysis by using mathematical tools~\cite{zhou2019evolutionary}. In~\myref{eq:fitness}, $i\in\{0,1,...,\frac{n}{2}\}$ is the number of green triangles (including the triangles that are contained in the green sector regions) that have been correctly classified by $x$, and $j\in\{0,1,...,\frac{3n}{4}\}$ is the subtraction difference between the number of green segment regions correctly classified and the number of white segment regions misclassified. When a solution satisfies $n_B+n_C\geq b+c, n_B\leq b, n_A\geq a+b-n_B$, there are $i=\frac{n}{4}$ and $j=\frac{3n}{4}$, and this solution is an optimal solution with a fitness value of 1. Additionally, we find that when $n$ is greater than or equal to 6, the fitness improvement brought by maximizing the value of $j$ is significantly less than that by increasing the value of $i$ by 1, i.e., $\frac{3n}{4}\cdot (\frac{\pi}{n} - \frac{1}{2} \sin(\frac{2\pi}{n})) < \frac{1}{2} \sin(\frac{2\pi}{n})$. 

Thus, according to the number of B/C-type blocks contributing to the fitness, the entire solution space can be partitioned into $(b+c+1)$ subspaces that are denoted as $\cup_{i=0}^{b+c}\chi_i$. The following relationship exists between the subspaces:
\begin{equation*}
	\chi_0 < _f\chi_1 < _f\chi_2 < ... <_f\chi_{b+c},
\end{equation*}
where $\chi_i<_f\chi_{i+1}$ represents that $f(x)<f(y)$ for any $x\in \chi_i$ and $y\in \chi_{i+1}$. Note that the solutions in subspace $\chi_{b+c}$ may have more than $(b+c)$ B/C-type blocks, but the exceeding blocks do not affect fitness.

Furthermore, based on $j$ defined in Lemma~\ref{lemma:fitness}, we can partition each subspace $\chi_i$ into $(a+b+1)$ sub-subspaces, and each of the sub-subspace can be denoted as $\chi_i^j$, where $i\in\{0,1,...,b+c\}$ and $j \in \{0,1,..., a+b\}$. The following relationship exists between different sub-subspaces of $\chi_i$:
\begin{equation*}
	\chi_i^0 < _f\chi_i^1 < _f\chi_i^2 < ... <_f\chi_i^{a+b},
\end{equation*}
where $\chi_i^j<_f\chi_i^{j+1}$ represents that $f(x)<f(y)$ for any $x\in \chi_i^j$ and $y\in \chi_i^{j+1}$.

Based on the partitioning of the solution space described above, we find that the optimal solutions are in sub-subspace $\chi_{b+c}^{a+b}$. Next, we will discuss the expected runtime when the (1+1)-ENAS algorithm finds the optimum for the first time.

\section{Runtime Analysis} \label{sec:runtime}

In this section, we analyze the expected runtime of the (1+1)-ENAS algorithm when using \new{the local} mutation and \new{global} mutation for solving the \textsc{Uniform} problem. As the (1+1)-ENAS algorithm generates only one offspring solution in each iteration (generation), its expected runtime is just equal to the expected number $T$ of generations to reach the optimum, which is denoted as $\mathbb{E}[T]$.

\subsection{\textcolor{myColor}{Local} Mutation}
We prove in Theorems~\ref{theorem:1+1_onebit_upper} and~\ref{theorem:1+1_onebit_lower} that the upper and lower bounds on the expected runtime of the (1+1)-ENAS algorithm with \new{local} mutation solving \textsc{Uniform} is $O(n)$ and $\Omega(\new{n})$, respectively.

\begin{theorem}
	\label{theorem:1+1_onebit_upper}
	The (1+1)-ENAS algorithm with \new{local} mutation and neurons parameter training strategy (defined in Definition~\ref{def:fixed_parameters_strategy}) needs $\mathbb{E}[T]\leq 63n/4 \in O(n)$ to find the optimum of the \textsc{Uniform} problem with $D=2$, where $n\geq 6$. 
\end{theorem}

The main idea of the proof is based on the fitness-level technique~\cite{wegener2003methods,sudholt2013new}. By analyzing the expected runtime required to jump out of each level and summing them up, we can determine the overall runtime. Based on the jumps that can be achieved by \new{local} mutation, the sub-subspaces can be categorized into three scenarios:
\begin{itemize} 
	\item For a solution in $\chi_i^{j}$, where $i<{b+c}$ and $j<a+b$, it can jump to one of the sub-subspaces $\{\chi_i^{j+1}, \chi_{i+1}^{j-2}, \chi_{i+1}^{j-1}, \chi_{i+1}^{j}, \chi_{i+1}^{j+1}\}$ by mutation.
	\vspace{-0.5em}
	\item For a solution in $\chi_i^{j}$, where $i<{b+c}$ and $j=a+b$, it can jump to sub-subspace $\chi_{i+1}^{j-1}$ or $\chi_{i+1}^{j}$ by mutation.
	\vspace{-0.5em}
	\item For a non-optimal solution in $\chi_{a+b}^j$, it can only jump to sub-subspace $\chi_{a+b}^{j+1}$ by mutation.
\end{itemize}
We proceed to establish the proof for Theorem~\ref{theorem:1+1_onebit_upper}. 
\begin{proof}[Proof of Theorem~\ref{theorem:1+1_onebit_upper}]
	The optimization process for solving the \textsc{Uniform} problem can be divided into two phases: 1) searching for a solution with $(b+c)$ B/C-type blocks that can partition all of the green triangles, i.e., finding a solution with $i=b+c$; 2) searching for a solution that can correctly partition all green and white segment regions, i.e., finding a solution with $i=b+c$ and $j=a+b$. Let $\mathbb{E}[T_1]$ and $\mathbb{E}[T_2]$ represent the expected runtime for each phase, respectively, so that we have $\mathbb{E}[T]=\mathbb{E}[T_1]+\mathbb{E}[T_2]$.
	
	
	In the first optimization phase, each parent satisfies $n_B+n_C<b+c$. The solution space can be partitioned into $(b+c+1)$ subspaces, i.e., $\cup_{i=0}^{b+c}\chi_i$. We say that algorithm $\mathcal{A}$ is in level $i$ if the best individual created so far is in $\chi_i$. Let $s_i$ be the lower bound on the probability of $\mathcal{A}$ creating a new solution in $\cup_{i'=i+1}^{b+c}\chi_{i'}$, provided $\mathcal{A}$ is in $\chi_i$. Since $\mathcal{A}$ uses the \new{local} mutation operator, the most optimistic assumption is that $\mathcal{A}$ makes it rise one level (i.e., $\chi_{i+1}$). A successful jump from $\chi_{i}$ to $\chi_{i+1}$ can occur in two ways: 1) when $n_A>0$, modifying an A-type block into a B/C-type with the probability of $(1/3)\cdot(1/3)=1/9$, where the first $1/3$ is the probability of selecting Modification operator, and the second $1/3$ is the probability of choosing to change A-type block; 2) adding a B/C-type block with probability $(1/3)\cdot (2/3)=2/9$, where $1/3$ is the probability of selecting Addition operator, and $2/3$ is the probability of adding a B/C-type block. Thus, it can be derived that $s_i=2/9$. Consequently, the expected runtime for the first phase is
	\vspace{-0.5em}
	\begin{equation*}
		\label{eq:addBC}
		\mathbb{E}[T_1]
			\resizebox{!}{!}{$\leq\sum\nolimits_{i=0}^{b+c-1} {(s_i)^{-1}}
			= (\frac{2}{9})^{-1} \cdot (b+c) 
			= \frac{9n}{4} \in O(n).$}
	\end{equation*}
	
	After the aforementioned phase, we pessimistically assume that $\mathcal{A}$ belongs to $\chi_{b+c}^0$. The potential solutions in the subsequent optimization process can be further divided into $(a+b+1)$ sub-subspaces, i.e., $\cup_{j=0}^{a+b}\chi_{b+c}^j$. The realization of this optimization phase relies on algorithm $\mathcal{A}$ increasing the value of $j$ through mutation operations. To increase $j$, $\mathcal{A}$ should decrease the number of B-type blocks that have incorrectly classified the white segment regions (i.e., decreasing $\max\{0,\min\{(n_B-b),(c-n_C)\}\}$) or increase the number of A/B-type blocks that can correctly classify the green segment regions (i.e., increasing $\min\{n_A, \new{a+\max\{0,b-n_B\}}\} + \min\{n_B,b\}$). Thus, we consider three cases:
	
	(1) If $n_B>b$ and $n_C<c$, there are $\min\{n_B-b,c-n_C\}>0$ white segment regions that are misclassified. Modifying a B-type block into C-type enables accurate classification of the white segment region, which happens with probability $(1/3)\cdot (1/3)\cdot (1/2)=1/18$, where the first $(1/3)$ is the probability of selecting Modification operator, and $(1/3)\cdot(1/2)$ is the probability of selecting B-tyep block and modifying it into C-type block. In addition, it also can add a C-type block with probability $1/9$ to classify the white segment region correctly. Thus, the probability of the fitness improvement in this case is at least $(1/18)+(1/9)=1/6$.
	

	(2) If $n_B\leq b$, there are $a+b-(n_A+n_B)$ green segment regions that are not classified correctly. For this case, the algorithm can increase the classified green segment regions in four ways: 1) by adding an A-type block with probability $1/9$; 2) by adding an B-type block when $n_B<b$, whose probability is $1/9$; 3) by modifying a C-type block into a B-type when $n_C>0$ with probability $1/18$; 4) by modifying a C-type block into A-type when $n_B+n_C>b+c$ with probability $1/18$. However, the latter three ways require the solution to satisfy additional conditions. Thus, the probability of the fitness improvement in this case is at least $1/9$.

	(3) If $n_B>b$ and $n_C\geq c$, there are $a-n_A$ green segment regions that are not classified correctly. For this case, the algorithm can increase the classified green segment regions in three ways: 1) by adding an A-type block with probability $1/9$; 2) by modifying a B-type block into A-type with probability $1/18$; 3) by modifying a C-type block into A-type when $n_C>c$ with probability $1/18$. Thus, the probability of the fitness improvement in this case is at least $(1/9)+(1/18)=1/6$.

	Based on the above analysis, a jump upgrade between any two sub-subspaces (i.e., $\chi_{b+c}^j$ and $\chi_{b+c}^{j+1}$) occurs with probability at least ${1}/{9}$. Thus, the expected runtime required to transition from $\chi_{b+c}^{0}$ to $\chi_{b+c}^{a+b}$ is at most $9(a+b)={27n}/{4}\in O(n)$, i.e., $\mathbb{E}[T_2] \in O(n)$.

	By combining the two phases, we get $\mathbb{E}[T] \in O(n)$.
\end{proof}


The proof of the lower bound on the expected runtime, as shown in Theorem~\ref{theorem:1+1_onebit_lower}, relies on the additive drift analysis tool~\cite{he2001drift}. First, a distance function has to be constructed to measure the current distance to the optimum. We define two distance functions for the two optimization phases. Let $d_1(x)=(b+c)-i$ and $d_2(x)=(a+b)-j$ represent the difference between a solution and the optimum in terms of $i$ and $j$ (defined in Lemma~\ref{lemma:fitness}), where $i$ is the number of B/C-type blocks contributing to the fitness in solution $x$, and $j$ is the number of A/B-type blocks contributing to the fitness minus the number of B-type blocks that negatively affect the fitness. Then, we need to analyze the improvement in the distance to target state space (i.e., $\chi_{b+c}$ and $\chi_{b+c}^{a+b}$ for each phase) by one step. Finally, the upper/lower bound on the expected runtime can be derived through dividing the initial distance by a lower/upper bound on the improvement, where the calculation of initial distance relies on Lemma~\ref{lemma:EZ0}, i.e., the expected number of B/C-type blocks of the initial solution.

\begin{theorem}
	\label{theorem:1+1_onebit_lower}
	When the upper bound $s$ on the number of each type of block in the initial solution is $n/4$, the (1+1)-ENAS algorithm with \new{local} mutation and neurons parameter training strategy (defined in Definition~\ref{def:fixed_parameters_strategy}) needs $\mathbb{E}[T]=\Omega(\new{n})$ to find the optimum of the \textsc{Uniform} problem with $D=2$, where $n\geq 6$. 
\end{theorem}

\begin{lemma}
    \label{lemma:EZ0}
    For the (1+1)-ENAS algorithm, let $Z_0=n_B[x_0]+n_C[x_0]$ denote the number of B/C-type blocks of the initial solution $x_0$, where $n_B[x_0] \sim U[0,s]$ and $n_C[x_0]\sim U[0,s]$. The expectation of $Z_0$ is $s$.
\end{lemma}
\begin{proof}
    The probability distribution of $Z_0$ is $P(Z_0=z)= {(\min\{z,2s-z\}+1)}/{(s+1)^2}$,
    where $z\in[0,1,...,2s]$. Thus, the expectation of $Z_0$ is
	\vspace{-0.5em}
    \begin{equation}
        \label{eq:E_Z0-upper}
        \begin{aligned}
            \mathbb{E}[Z_0]
            & \resizebox{!}{!}{ $=\sum\nolimits_{z=0}^{2s} z\cdot P(Z_0=z) $}
			\\
            &  \resizebox{!}{!}{ $=\sum\nolimits_{z=0}^{s-1} \frac{(z+2s-z)(z+1)}{(s+1)^2}  +  \frac{s}{s+1} = s
            , $}
			\vspace{0.5em}
        \end{aligned}
    \end{equation}
	where the second equality holds because the probability distribution diagram of $Z_0$ is symmetrically distributed about the point $P(Z_0=s)=1/(s+1)$.
\end{proof}

\begin{proof}[Proof of Theorem~\ref{theorem:1+1_onebit_lower}]
{
\color{myColor}
    We begin by analyzing the first phase. By considering operations that decrease the distance value, we can get the upper bound of the expected one-step improvement $\Delta_1$ of the first phase, i.e., $\Delta_1 \leq (1/3)\cdot P(\text{add a B/C-type block}) +(1/3)\cdot P(\text{modify an A-type block into a B/C-type})= (1/3)\cdot(2/3)+(1/3)\cdot(1/3)\cdot 1=1/3$.
}
    
    Let $Z_0$ denote the number of B/C-type blocks of the initial solution $x_0$. According to Lemma~\ref{lemma:EZ0} and $s\leq n/4=(b+c)/2$, we can get that $\mathbb{E}[Z_0]= s \leq {(b+c)}/{2}$. Furthermore, let $Z'_0$ represent the number of missing B/C-type blocks in $x_0$, i.e., $Z'_0=(b+c)-Z_0\new{\geq (b+c)/2= n/4}$. \new{Thus, according to the additive drift analysis tool, the expected runtime for the first phase is at most $(n/4)/(1/3)\in\Omega(n)$.}

    Because the runtime $T_1$ of the first phase is a lower bound on the runtime $T$ of the whole process for finding an optimal solution belonging to $\chi_{b+c}^{a+b}$, we have $\mathbb{E}[T]= \Omega(\new{n})$.
 \end{proof}

\subsection{\textcolor{myColor}{Global} Mutation} \label{sec:runtime_qbit}
Next, we will show that by replacing \new{local} with \new{global} mutation, the (1+1)-ENAS algorithm \new{can use the same time} to find the optimal architecture. In particular, we prove in Theorems~\ref{theorem:1+1_qbit_upper} and~\ref{theorem:1+1_qbit_lower} that the expected runtime of the (1+1)-ENAS algorithm with \new{global} mutation is $O(n)$ and $\Omega(n)$ for solving the \textsc{Uniform} problem, respectively.


\begin{theorem}
	\label{theorem:1+1_qbit_upper}
	When the upper bound $s$ on the number of each type of block in the initial solution is $n/4$, the (1+1)-ENAS algorithm with \new{global} mutation and neurons parameter training strategy (defined in Definition~\ref{def:fixed_parameters_strategy}) needs $\mathbb{E}[T]= O(n)$ to find the optimum of the \textsc{Uniform} problem with $D=2$, where $n\geq 6$. 
\end{theorem}

\begin{theorem}
	\label{theorem:1+1_qbit_lower}
	When the upper bound $s$ on the number of each type of block in the initial solution is $n/4$, the (1+1)-ENAS algorithm with \new{global} mutation and neurons parameter training strategy (defined in Definition~\ref{def:fixed_parameters_strategy}) needs $\mathbb{E}[T] = \Omega(n)$ to find the optimum of the \textsc{Uniform} problem with $D=2$, where $n\geq 6$. 
\end{theorem}

    The proof of Theorems~\ref{theorem:1+1_qbit_upper} and~\ref{theorem:1+1_qbit_lower} is similar to that of Theorem~\ref{theorem:1+1_onebit_upper} and~\ref{theorem:1+1_onebit_lower}, i.e., dividing the optimization process into two phases. Each phase aims to find $x_{T_1}\in \chi_{b+c}$ and $x_{T_2}\in \chi_{b+c}^{a+b}$, respectively. Differently, the analysis of jumps is more sophisticated, because jumps across multiple layers (subspaces) can be achieved in each iteration when using the \new{global} mutation. 
\begin{proof}[Proof of Theorem~\ref{theorem:1+1_qbit_upper}] 
	We begin by analyzing the first phase. After one mutation, the distance $d_1$ may either increase, decrease, or remain unchanged. Let $D_{1}^{+}$, $D_{1}^{-}$, and $D_{1}^{N}$ denote these events, respectively. The probabilities of these events occurring are as follows: 
	(1) $2/9=P($add B/C-type block$) \leq P(D_{1}^{-})\leq P($add B/C-type block, or modify A-type block to B/C-type$)=1/3$, where the lower bound is held by the case that there is no A-type block; 
	(2) $1/6 \leq P(D_{1}^{N})= P($modify a B-type block into C-type, or modify a C-type into B-type, or add/remove an A-type block$)\leq 1/3$, where the lower bound is based on the fact that there is at least one B or C-type block in any generation but may not have an A-type block;
	(3) $P(D_{1}^{+})=1-P(D_{1}^{-})-P(D_{1}^{N})$, thus $P(D_{1}^{+})\geq 1/3$.
	Then, we can get the lower bounds on the probabilities $P(D_{1}^{-})$, $P(D_{1}^{+})$, and $P(D_{1}^{N})$, which are abbreviated as $l_{1}^{-}=2/9$, $l_{1}^{+}=1/3$, and $l_{1}^{N}=1/6$, respectively. 
	

    Now we analyze the expected one-step improvement $\Delta_1$ of the first phase. Let $\Delta(k)$ denote the expected one-step improvement, given that the number of mutation times is $k$ in an iteration. Then, we have $\Delta(k)\geq \sum_{u=1}^{k}\sum_{v=0}^{\min\{u-1,k-u\}} {k\choose u} {{k-u}\choose v}\cdot (l_{1}^{-})^u (l_{1}^{+})^v (l_{1}^{N})^{k-u-v} \cdot (u-v)$, where $u$, $v$, and $(k-u-v)$ are the number of mutation times in which the distance decreases by one, the distance increases by one, and the distance remains unchanged, respectively. Thus, we have
	\begin{equation}
		\label{eq:deta_1}
		\begin{aligned}
			\Delta_1 &  
			\resizebox{!}{!}{$
				= \sum_{k=1}^{+\infty} P(K=k)  \Delta (k) 
				\geq \sum_{k=1}^{3} P(K=k) \Delta (k)
			$}
			\\ &
			\resizebox{!}{!}{$
				\geq \frac{1}{e} (l_{1}^{-}) 
				+ \frac{1}{e}\sum\nolimits_{u=1}^{2}{2\choose u} (l_{1}^{-})^u (l_{1}^{N})^{2-u} u
			$}
			\\ & \resizebox{0.94\linewidth}{!}{$
				+ \frac{1}{2e}\sum\nolimits_{u=1}^{3}{3\choose u} (l_{1}^{-})^u  \sum\nolimits_{v=0}^{m} \binom{3-u}{v} (l_{1}^{+})^v (l_{1}^{N})^{3-u-v} (u-v)
			$}\\
			& \resizebox{!}{!}{$
				\geq \frac{1}{e}\cdot [\frac{2}{9}+\frac{14}{81}+\frac{73}{972}]
				>\frac{1}{6}
			$}
			,
		\end{aligned}
	\end{equation}
	where $P(K=1)={1}/{e}$, $P(K=2)={1}/{e}$, and $P(K=3)={1}/{(2e)}$ since $K-1\sim Pois(1)$, and $m=\min\{u-1,3-u\}$. 

	
	Owing to the expected number of B/C-type blocks for the initial solution $x_0$ is $s$ as derived in Lemma~\ref{lemma:EZ0} and $0\leq s\leq n/4$, the expected distance for $x_0$ is $b+c-s\leq b+c=n/2$. Consequently, according to the additive drift analysis tool, the expected runtime for the first phase is at most $(n/2)/(1/6) = 3n \in O(n)$, i.e., $\mathbb{E}[T_1] \in O(n)$.
	
		
	Then, we analyze the second phase. To make progress, it is necessary to maintain $n_B+n_C\geq b+c$ (i.e., $d_1$ unchanged). Thus, it is sufficient to consider the change in $d_2$ to analyze the lower bound of one-step progress. Let $D_{2}^{+}$, $D_{2}^{-}$, and $D_{2}^{N}$ denote the events where the distance from $d_2$ increases, decreases, and remains unchanged, respectively. The probabilities of these events occurring are bounded by: 
	(1) $P(D_{2}^{-})=P($add an A-type block, or modify B-type block into C-type, or add an C-type block$)\geq 1/9$, where the lower bound is held by the case of adding an A-type block when $n_B=b$; 
	(2) $P(D_{2}^{N})= P($add a redundant A/B/C-type block, or remove a useful A/B/C-type block$)\geq 1/9$, where the lower bound is held by the case of adding a C-type block when $n_B<b$ and $n_A+n_B<a+b$;
	(3) $P(D_{2}^{+})= P($wrong modification between A/B/C-type blocks, or remove a useful A/B-type block$)\geq 1/18$, where the lower bound is held by the case of modifying B-type into C-type block when $n_B<b$.
	
	Similar to the calculation of $\Delta_1$ in \myref{eq:deta_1}, we can get the expected one-step improvement of the second phase by
	\begin{equation*}
		\begin{aligned}
			\resizebox{!}{!}{
				$
				\Delta_2 
				\geq \sum\nolimits_{k=1}^{3} P(K=k) \Delta(k)
				\geq \frac{1}{e}\cdot [\frac{1}{9} + \frac{4}{81} + \frac{11}{972}]
				> \frac{1}{16}
				$
			}
			.
		\end{aligned}
	\end{equation*}
	In the most pessimistic scenario, there is $d_2(x_{T_1})=a+b$. Thus, the expected runtime for the second phase is at most ${(a+b)}/(1/16) = 12n \in O(n)$, i.e., $\mathbb{E}[T_2]\in O(n)$.
	Combining the two phases, we get $\mathbb{E}[T] \in O(n)$.  
\end{proof}

\begin{proof}[Proof of Theorem~\ref{theorem:1+1_qbit_lower}]
	Similar to the proof of Theorem~\ref{theorem:1+1_qbit_upper}, we utilize its definition of $k,u,v,P(D_{1}^{-}), P(D_{1}^{+})$, and $P(D_{1}^{N})$, and analyze the expected one-stpe improvement $\Delta(k)$. 

	Note that the upper bound of $P(D_{1}^{+})$ is $1/3$ which is held by the case of removing a B/C-type block or modifying B/C-type block into A-type. Thus, the upper bounds of $P(D_{1}^{-}), P(D_{1}^{+}),$ and $P(D_{1}^{N})$ are all less than or equal to ${1}/{3}$, implying that the one-step improvement of the first optimization phase is bounded by $\Delta(k)\leq \sum_{u=1}^{k} \sum_{v=0}^{\min\{u-1,k-u\}} {k\choose u} {{k-u}\choose v} (u-v) (1/3)^k \leq \sum_{u=1}^{k}{k\choose u} 2^{k-u} k/3^k$, where the last inequality holds by $u-v\leq k$. Then, we have the expected one-step improvement of the first phase by
	\begin{equation*}
		\label{eq:qbit_deta1_upper}
		\begin{aligned}
			\Delta_1 
			& \resizebox{!}{!}{$
			\leq \sum\nolimits_{k=1}^{+\infty} P(K=k) \cdot \sum\nolimits_{u=1}^{k}{k\choose u} 2^{k-u} \cdot (k/3^k)
			$}\\
			& \resizebox{!}{!}{$
			= \sum\nolimits_{k=1}^{+\infty} P(K=k) \cdot ( \left(1+2\right)^k-{k\choose 0}\cdot2^k) \cdot (k/3^k)
			$}\\
			& \resizebox{!}{!}{$
			= \sum\nolimits_{k=1}^{+\infty} P(K=k)\cdot k\cdot (1-\left(\frac{2}{3}\right)^k) 
			$}
			\\
			& \resizebox{!}{!}{$
				= \sum_{k=1}^{+\infty} k \cdot \frac{\lambda^{k-1}}{(k-1)!} e^{-\lambda} - \sum_{k=1}^{+\infty} k (\frac{2}{3})^k \cdot \frac{\lambda^{k-1}}{(k-1)!} e^{-\lambda}
			$}
			\\
			& \resizebox{!}{!}{$
				= 2- (\frac{2}{3}) e^{(\frac{2}{3}-1)\lambda} \sum_{k=1}^{+\infty} k\cdot \frac{(\frac{2}{3}\lambda)^{k-1}}{(k-1)!} e^{-\frac{2}{3}\lambda}
			$}
			\\ & \resizebox{!}{!}{$=2 - \frac{2}{3}e^{-1/3}\cdot(\frac{2}{3}+1) = 2 - \frac{10}{9}e^{-{1}/{3}}$}
			.
		\end{aligned}
	\end{equation*}
	
	Because the expected number of B/C-type blocks of the initial solution $x_0$ is $s$ as derived in Lemma~\ref{lemma:EZ0}, the expectation of $d_1(x_0)$ is $b+c-s \geq n/4$. Thus, the expected runtime of $T_1$ is bounded by 
	\begin{equation*}
		\begin{aligned}
			\mathbb{E}[T_1] \geq \frac{{n}/{4}}{2 - \frac{10}{9}e^{-{1}/{3}}} \geq \frac{1}{5} n \in \Omega(n),
		\end{aligned}
	\end{equation*}
	implying that the expected runtime of the whole process is $\mathbb{E}[T]= \Omega(n)$.
\end{proof}

\section{Experiments}


\begin{figure}
	\centerline{\includegraphics[width=1\linewidth]{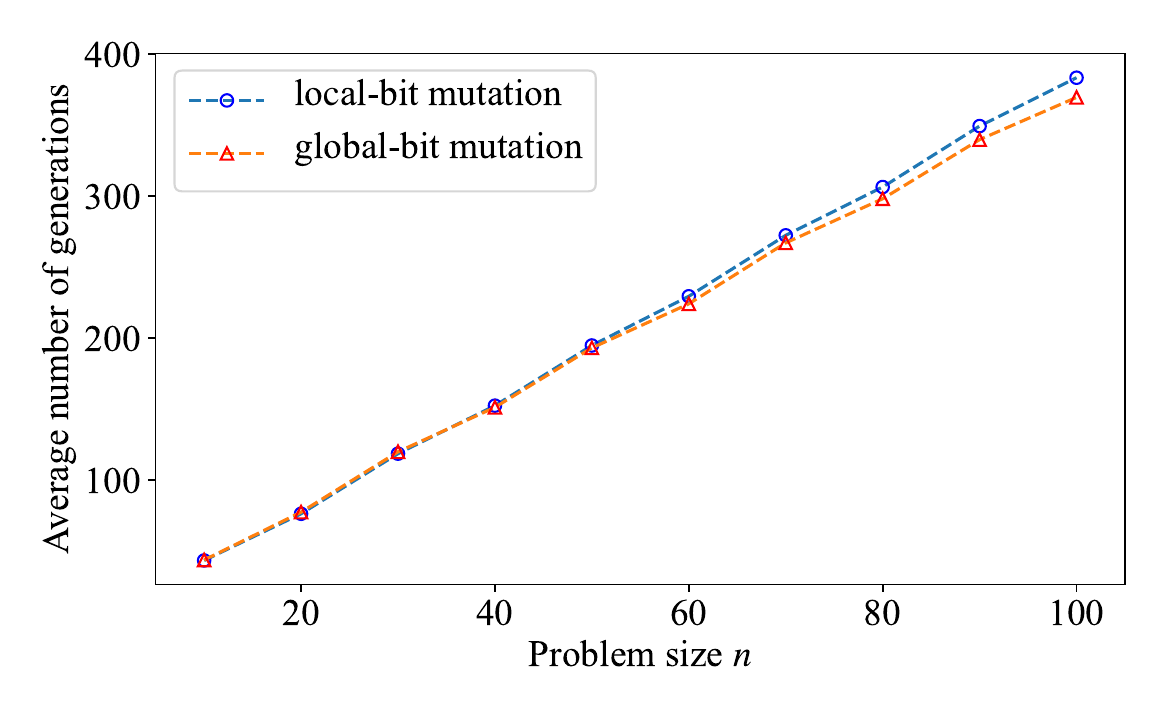}}
	\caption{Average number of generations of the (1+1)-ENAS algorithm using \textcolor{myColor}{local} and \textcolor{myColor}{global} mutations for solving \textsc{UNIFORM}.}
	\label{fig:exp}
\end{figure}

In this section, we investigate the empirical performance of the (1+1)-ENAS algorithm with \new{local} mutation and \new{global} mutation by experiments. We set the problem size $n$ from 10 to 100, with a step of 10, the upper bound $s=n/4$ for each type of block, as suggested in Theorems~\ref{theorem:1+1_onebit_upper} to \ref{theorem:1+1_qbit_lower}. Figure~\ref{fig:exp} shows the average number of generations (on running 10,000 runs) of the algorithms before reaching the optimal value of \textsc{UNIFORM}. We can observe that \new{local} mutation and \new{global} mutations achieve nearly the same performance, which verifies our theoretical analysis. 

\section{Conclusion} \label{sec:conclusion}


In this paper, we take the first step towards the mathematical runtime analysis of the ENAS algorithm. Specifically, we defined a problem \textsc{UNIFORM} based on a binary classification problem with labels in $\{0,1\}$ on a unit \new{circle} points. Furthermore, we formulated an explicit fitness function to represent the relationship between the neural architecture and classification accuracy. Then, we analyzed the expected runtime of the (1+1)-ENAS algorithms with \new{local} and \new{global} mutations, respectively. Our theoretical results show that \new{local} mutation and \new{global} mutations can achieve nearly the same performance, which is also verified by experiments. In this first study of the runtime of ENAS algorithms, we have only scratched the surface of the rich structure arising from a binary classification problem. Future works will conclude more advanced classification problems. In addition, more advanced evolutionary operators (e.g., crossover operators) and population mechanisms will be considered for future research.

\bibliographystyle{named}
\bibliography{ijcai24}

\clearpage
\section{Appendix}
{
\color{myColor}
This part provides the proof of Lemma~\ref{lemma:fitness}, which is omitted in Section~\ref{sec3:fitness_ij} of our original paper due to space limitation.

\begin{proof} [Proof of Lemma 1]

    The $\textsc{Uniform}$ problem evenly divides the unit circle into $n$ sectors and each sector can be divided into a triangle with area $s_{\mathrm{tri}}=\frac{1}{2}\sin(\frac{2\pi}{n})$ and a segment region with area $s_{\mathrm{seg}}=(\frac{\pi}{n}-\frac{1}{2}\sin(\frac{2\pi}{n}))$. Given a neural architecture $x$ with $n_A$ A-type blocks, $n_B$ B-type blocks, and $n_C$ C-type blocks, its fitness (classification accuracy) is determined by the area of correctly and incorrectly classified regions. We begin our proof by introducing the impact of parameter tuning on classification accuracy.
    
    According to Definition~\ref{def:fixed_parameters_strategy}, the parameter tuning is exclusively associated with the $\varphi$ of each hyperplane (line). Specifically, the parameters that need to be tuned are the $\varphi$ of the neuron in A-type block, the $\varphi$ of one neuron in B-type block, and the $\varphi$ of one neuron in C-type block. Let the boundary between the green and white region be called the green boundary. To achieve optimal classifying, each hyperplane of $x$ should overlap with a green boundary while its normal vector points to the green region. A special case is that when a C-type block is used to classify the triangle of the green sector region or when an A-type block is used to classify the segment region of the green sector region, the hyperplane of this block needs to overlap the edge of the triangle that is not connected to the origin.
    Note that a region will not be classified by multiple blocks unless $n_A> a+\max\{0,(b-n_B)\}$ or $n_B+n_C\geq b+c$, since the training strategy ensures optimal parameter configuration. Let $l_t$ and $l_s$ be the number of correctly classified triangles and segment regions, respectively. The classification accuracy equals $(l_t\cdot s_{\mathrm{tri}} + l_s\cdot s_{\mathrm{seg}})/{\pi}$. We explain $l_t$ and $l_s$ in detail below. 
    

    First, we analyze the number of correctly classified triangles, i.e., $l_t$, which is influenced by the number of B/C-type blocks. Under the optimal parameter guarantee, the white triangles will not be incorrectly classified by any block. Thus, the number of white triangles that are correctly classified is $n/2$. Furthermore, according to the different cases of $n_B$ and $n_C$, we analyze the number of green triangles that are correctly classified.

    (1) Case 1: When $n_B\geq b$ and $n_C\geq c$, there are $b$ B-type blocks for classifying the green sectors and $c$ C-type blocks for classifying the green triangles connected with white segment regions; thus, $b+c$ green triangles are correctly classified.
    
    (2) Case 2: When $n_B\geq b$ and $n_C<c$, there are $b$ B-type blocks for classifying the green sectors, while $n_C$ C-type blocks and $\min\{n_B-b,c-n_C\}$ B-type blocks for classifying the green triangles connected with white segment regions; thus, $b+ n_C+\min\{n_B-b,c-n_C\}=\min\{n_B+n_C,b+c\}$ green triangles are correctly classified.

    (3) Case 3: When $n_B< b$ and $n_C\geq c$, there are $c$ C-type blocks for classifying the green triangles connected with white segment regions, while $n_B$ B-type blocks and $\min\{b-n_B,n_C-c\}$ C-type blocks for classifying the green sectors; thus, $c+n_B+\min\{b-n_B,n_C-c\}=\min\{c+b,n_B+n_C\}$ green triangles are correctly classified.
    
    (4) Case 4: When $n_B< b$ and $n_C<c$, there are $n_B$ B-type blocks for classifying the green sectors and $n_C$ C-type blocks for classifying the green triangles connected with white segment regions; thus, $n_B+n_C$ green triangles are correctly classified.

    According to the above four cases, the number of green triangles that are correctly classified is $\min\{b+c,n_B+n_C\}=i$. Thus, $l_t= n/2+i$.

    Second, we analyze the number of correctly classified segment regions, i.e., $l_s$, which is affected by the number of A/B-type blocks. Since B-type blocks may incorrectly classify white segment regions when $n_B>b$ and $n_C<c$, we consider two cases based on $n_B$. 

    
    (1) Case 1: When $n_B\leq b$, $n_B$ green segment regions that are connected to the green triangle are classified correctly by B-type blocks. Furthermore, there are $\min\{n_A,a+b-n_B\}$ green segment regions (including the green segment region in the green sector) that are correctly classified by A-type blocks. In addition, the number of white segment regions that are correctly classified is the number of white segment regions in \textsc{Uniform}, i.e., $n/4$. Thus, we can get $l_s=n_B+\min\{n_A,a+b-n_B\}+n/4$ for this case.
    
    (2) Case 2: When $n_B>b$, $b$ green segment regions that are connected to the green triangle  are classified correctly, while $\min\{(n_B-b),(c-n_C)\}$ white segment regions that are connected to the green triangle  are classified incorrectly when $n_C<c$. Thus, the number of white segment regions that are correctly classified is $ n/4 - \max\{0,\min\{(n_B-b),(c-n_C)\}\}$. Furthermore, $\min\{n_A,a\}$ green segment regions that are connected to the white triangle  will be classified correctly by A-type blocks. Thus, we can get $l_s=b+ n/4 - \max\{0,\min\{(n_B-b),(c-n_C)\} + \min\{n_A,a\}\}$ for this case.

    Based on Case 1 and Case 2, we can get $l_s =\min\{n_B,b\} + \min\{n_A,a+\max\{0, b-n_B\}\} + n/4 - \max\{0,\min\{(n_B-b),(c-n_C)\}\} = n/4 + j$, where $j$ is the number of A/B-type blocks that classify the green segment regions (i.e., $\min\{n_B,b\} + \min\{n_A,a+\max\{0, b-n_B\}\}$) minus the number of B-type blocks that classify the white segment regions connected to green triangles(i.e., $\max\{0,\min\{(n_B-b),(c-n_C)\}\}$). 
\end{proof}
}

\end{document}